\documentclass{article}

\usepackage{times}
\usepackage{graphicx} 
\usepackage{subcaption}
\usepackage[numbers]{natbib}

\usepackage{color}
\usepackage{wrapfig}
\usepackage{amssymb,amsfonts}
\usepackage{amsmath,amssymb}
\usepackage{float}
\usepackage{comment}
\usepackage{verbatim}
\usepackage{appendix,stfloats}
\usepackage{bbm}
\usepackage{blkarray}
\usepackage{verbatim}
\usepackage{algorithm}
\usepackage{multirow}
\usepackage[noend]{algpseudocode}
\makeatletter
\def\BState{\State\hskip-\ALG@thistlm}
\makeatother

\newtheorem{theorem}{Theorem}[section]
\newtheorem{lemma}[theorem]{Lemma}

\newenvironment{proof}[1][Proof]{\begin{trivlist}
\item[\hskip \labelsep {\bfseries #1}]}{\end{trivlist}}

\newcommand\Algphase[1]{%
\vspace*{-.2\baselineskip}\Statex\hspace*{\dimexpr-\algorithmicindent-2pt\relax}\rule{1\textwidth}{0.4pt}%
\Statex\hspace*{-\algorithmicindent}\textbf{#1}%
\vspace*{-.2\baselineskip}\Statex\hspace*{\dimexpr-\algorithmicindent-2pt\relax}\rule{1\textwidth}{0.4pt}%
}
\newcommand{\indicator}[1]{\mathds{1}_{\left[ {#1} \right] }}
\usepackage{array,dsfont}
\def \bx{\mathbf x}
\def \cT{\mathcal{T}}
\usepackage{tikz,amsmath}
\usepackage{tikz-qtree}
\usetikzlibrary{calc}
\usetikzlibrary{arrows}
\usepackage{verbatim}
\usetikzlibrary{shapes.geometric}
\usepackage{mathtools}
\usepackage{kbordermatrix}

\usepackage[final]{nips_2016}

\usepackage[utf8]{inputenc} 
\usepackage[T1]{fontenc}    
\usepackage{hyperref}       
\usepackage{url}            
\usepackage{booktabs}       
\usepackage{amsfonts}       
\usepackage{nicefrac}       
\usepackage{microtype}      

\title{Pruning Random Forests for Prediction on a Budget}

\author{
  Feng Nan \\
  Systems Engineering\\
  Boston University\\
  \texttt{fnan@bu.edu} \\
  \And
  Joseph Wang \\
  Electrical Engineering \\
  Boston University\\
  \texttt{joewang@bu.edu} \\  
  \And
  Venkatesh Saligrama \\
  Electrical Engineering\\
  Boston University \\
  \texttt{srv@bu.edu} \\
}

\begin{document}

\maketitle

\begin{abstract} 
We propose to prune a random forest (RF) for resource-constrained prediction. We first construct a RF and then prune it to optimize expected feature cost \& accuracy. We pose pruning RFs as a novel 0-1 integer program with linear constraints that encourages feature re-use. We establish total unimodularity of the constraint set to prove that the corresponding LP relaxation solves the original integer program. We then exploit connections to combinatorial optimization and develop an efficient primal-dual algorithm, scalable to large datasets. In contrast to our bottom-up approach, which benefits from good RF initialization, conventional methods are top-down acquiring features based on their utility value and is generally intractable, requiring heuristics. 
Empirically, our pruning algorithm outperforms existing state-of-the-art resource-constrained algorithms. \vspace{-0.21in}
\end{abstract}

\section{Introduction}
%
Many modern classification systems, including internet applications (such as web-search engines, recommendation systems, and spam filtering) and security \& surveillance applications (such as wide-area surveillance and classification on large video corpora), face the challenge of prediction-time budget constraints~\cite{xu2013cost}. Prediction-time budgets can arise due to monetary costs associated with acquiring information or computation time (or delay) involved in extracting features and running the algorithm. 
We seek to learn a classifier by training on fully annotated training datasets that maintains high-accuracy while meeting average resource constraints during prediction-time. We consider a system that adaptively acquires features as needed depending on the instance(example) for high classification accuracy with reduced feature acquisition cost.

\noindent
We propose a two-stage algorithm. In the first stage, we train a random forest (RF) of trees using an impurity function such as entropy or more specialized cost-adaptive impurity \cite{icml2015_nan15}. Our second stage takes a RF as input and attempts to jointly prune each tree in the forest to meet global resource constraints. During prediction-time, an example is routed through all the trees in the ensemble to the corresponding leaf nodes and the final prediction is based on a majority vote. The total feature cost for a test example is the sum of acquisition costs of \emph{unique} features\footnote{When an example arrives at an internal node, the feature associated with the node is used to direct the example. If the feature has never been acquired for the example an acquisition cost is incurred. Otherwise, no acquisition cost is incurred as we assume that feature values are stored once computed.} acquired for the example in the entire ensemble of trees in the forest. 
\footnote{For time-sensitive cases such as web-search we parallelize the implementation by creating parallel jobs across all features and trees. We can then terminate jobs based on what features  are returned.}

We derive an efficient scheme to learn a \emph{globally optimal pruning} of a RF minimizing the empirical error and incurred average costs. We formulate the pruning problem as a 0-1 integer linear program that incorporates feature-reuse constraints. By establishing total unimodularity of the constraint set, we show that solving the linear program relaxation of the integer program yields the optimal solution to the integer program resulting in a \emph{polynomial time algorithm for optimal pruning}. We develop a primal-dual algorithm by leveraging results from network-flow theory for scaling the linear program to large datasets. Empirically, this pruning outperforms state-of-the-art resource efficient algorithms on benchmarked datasets.
\noindent
\begin{wraptable}{r}{83mm}
\setlength\tabcolsep{4pt}
\vspace{-.2cm}
\centering
    \begin{tabular}{|l|c|c|c|c|c|}
    \hline
    & No Usage & 1--7 & > 7 & Cost & Error \\ \hline
      Unpruned RF & 7.3\% & 91.7\% & 1\% & 42.0 & 6.6\% \\ \hline
       BudgetPrune & 68.3\% & 31.5\% & 0.2\% & 24.3 & 6.7\% \\ \hline
    \end{tabular}
    \caption{\footnotesize  Typical feature usage in a 40 tree RF before and after pruning (our algorithm) on the MiniBooNE dataset. Columns 2-4 list percentage of test examples that do not use the feature, use it 1 to 7 times, and use it greater than 7 times, respectively. Before pruning, 91\% examples use the feature only a few (1 to 7) times, paying a significant cost for its acquisition; after pruning, 68\% of the total examples no longer use this feature, reducing cost with minimal error increase. Column 5 is the average feature cost (the average number of unique features used by test examples). Column 6 is the test error of RFs. Overall, pruning dramatically reduces average feature cost while maintaining the same error level.
    } \label{table}
    \vspace{-.35cm}
\end{wraptable}
Our approach is motivated by the following considerations: \\ ({\bf i}) RFs are scalable to large datasets and produce flexible decision boundaries yielding high prediction-time accuracy. The sequential feature usage of decision trees lends itself to adaptive feature acquisition. \\({\bf ii}) RF feature usage is superfluous, utilizing features with introduced randomness to increase diversity and generalization. Pruning can yield significant cost reduction with negligible performance loss by selectively pruning features sparsely used across trees, leading to cost reduction with minimal accuracy degradation (due to majority vote). See Table \ref{table}. \\ ({\bf iii}) Optimal pruning encourages examples to use features either a large number of times, allowing for complex decision boundaries in the space of those features, or not to use them at all, avoiding incurring the cost of acquisition. It enforces the fact that once a feature is acquired for an example, repeated use incurs no additional acquisition cost. Intuitively, features should be repeatedly used to increase discriminative ability without incurring further cost. \\ ({\bf iv}) Resource constrained prediction has been conventionally viewed as a top-down (tree-growing) approach, wherein new features are acquired based on their utility value. This is often an intractable problem with combinatorial (feature subsets) and continuous components (classifiers) requiring several relaxations and heuristics. In contrast, ours is a bottom-up approach that starts with good initialization (RF) and prunes to realize optimal cost-accuracy tradeoff. Indeed, while we do not pursue it, our approach can also be used in conjunction with existing approaches.
\textbf{Related Work:} Learning decision rules to minimize error subject to a budget constraint during prediction-time is an area of recent interest, with many approaches proposed to solve the prediction-time budget constrained problem \cite{Fastmarginbasedcostsensitiveclassification_NanWTS14,AnLPSequentialLearningUnderBudgets_WangTS14,wang2014model,wang2014lp,Gao+Koller:NIPS11,DBLP:conf/icml/XuWC12,trapeznikov:2013b, NIPS2015_5982,ASTC_AAAI14}. These approaches focus on learning complex adaptive decision functions and can be viewed as orthogonal to our work. Conceptually, these are top-down ``growing'' methods as we described earlier (see ({\bf iv})). Our approach is bottom-up that seeks to prune complex classifiers to tradeoff cost vs. accuracy.

Our work is based on RF classifiers \cite{breiman}. Traditionally, feature cost is not incorporated when constructing RFs, however recent work has involved approximation of budget constraints to learn budgeted RFs \cite{icml2015_nan15}. The tree-growing algorithm in \cite{icml2015_nan15} does not take feature re-use into account. Rather than attempting to approximate the budget constraint during tree construction, our work focuses on pruning ensembles of trees subject to a budget constraint. Methods such as traditional ensemble learning and budgeted random forests can be viewed as complementary. 

Decision tree pruning has been studied extensively to improve generalization performance, we are not aware of any existing pruning method that takes into account the feature costs. A popular method for pruning to reduce generalization error is Cost-Complexity Pruning (CCP), introduced by Breiman et al. \cite{breiman1984classification}. CCP trades-off classification ability for tree size, however it does not account for feature costs. As pointed out by Li et al. \cite{Li2001DPP}, CCP has undesirable ``jumps" in the sequence of pruned tree sizes. To alleviate this, they proposed a Dynamic-Program-based Pruning (DPP) method for binary trees. The DPP algorithm is able to obtain optimally pruned trees of all sizes; however, it faces the curse of dimensionality when pruning an ensemble of decision trees and taking  feature cost into account. \cite{integerPruneZhang,OptimalPruning2009} proposed to solve the pruning problem as a 0-1 integer program; again, their formulations do not account for feature costs that we focus on in this paper. The coupling nature of feature usage makes our problem much harder. In general pruning RFs is not a focus of attention as it is assumed that overfitting can be avoided by constructing an ensemble of trees. While this is true, it often leads to extremely large prediction-time costs. Kulkarni and Sinha \cite{pruningRFSurvey} provide a survey of methods to prune RFs in order to reduce ensemble size. However, these methods do not explicitly account for feature costs.


\section{Learning with Resource Constraints}
In this paper, we consider solving the Lagrangian relaxed problem of learning under prediction-time resource constraints, also known as the error-cost tradeoff problem:
\begin{equation}\label{eq:budgetProb}
\min_{f \in \mathcal{F}}E_{(x,y)\sim\mathcal{P}}\left[err\left(y,f(x)\right)\right]+ \lambda E_{x\sim \mathcal{P}_x}\left[C\left(f,x\right)\right],
\end{equation}
where example/label pairs $(x,y)$ are drawn from a distribution $\mathcal{P}$; $err(y,\hat{y})$ is the error function; $C(f,x)$ is the cost of evaluating the classifier $f$ on example $x$; $\lambda$ is a tradeoff parameter. A larger $\lambda$ places a larger penalty on cost, pushing the classifier to have smaller cost. By adjusting $\lambda$ we can obtain a classifier satisfying the budget constraint. The family of classifiers $\mathcal{F}$ in our setting is the space of RFs, and each RF $f$ is composed of $T$ decision trees $\mathcal{T}_1,\dots,\mathcal{T}_T$. 

\noindent
{\bf Our approach:} Rather than attempting to construct the optimal ensemble by solving Eqn. \eqref{eq:budgetProb} directly, we instead propose a two-step algorithm that first constructs an ensemble with low prediction error, then prunes it by solving Eqn. \eqref{eq:budgetProb} to produce a pruned ensemble given the input ensemble. By adopting this two-step strategy, we obtain an ensemble with low expected cost while simultaneously preserving the low prediction error. 

There are many existing methods to construct RFs, however the focus of this paper is on the second step, where we propose a novel approach to prune RFs to solve the tradeoff problem Eqn.\eqref{eq:budgetProb}. 
Our pruning algorithm is capable of taking any RF as input, offering the flexibility to incorporate any state-of-the-art RF algorithm.

\section{Pruning with Costs} 
In this section, we treat the error-cost tradeoff problem Eqn. \eqref{eq:budgetProb} as an RF pruning problem. Our key contribution is to formulate pruning as a 0-1 integer program with totally unimodular constraints.

We first define notations used throughout the paper. A training sample $S=\{(\bx^{(i)},y^{(i)}):{i=1,\dots,N}\}$ is generated i.i.d. from an unknown distribution,  where $\bx^{(i)} \in \Re^K$ is the feature vector with a cost assigned to each of the $K$ features and $y^{(i)}$ is the label for the $i^{\mbox{th}}$ example. In the case of multi-class classification $y \in \{1,\dots,M\}$, where $M$ is the number of classes. Given a decision tree $\mathcal{T}$, we index the nodes as $h\in \{1,\dots,|\cT|\}$, where node $1$ represents the root node. Let $\tilde{\cT}$ denote the set of leaf nodes of tree $\cT$.
Finally, the corresponding definitions for $\mathcal{T}$ can be extended to an ensemble of $T$  decision trees $\{\mathcal{T}_t :t=1,\dots,T\}$ by adding an subscript $t$. 

\noindent
{\bf Pruning Parametrization:}
In order to model ensemble pruning as an optimization problem, we parametrize the space of all prunings of an ensemble.
The process of pruning a decision tree $\mathcal{T}$ at an internal node $h$ involves collapsing the subtree of $\cT$ rooted at $h$, making $h$ a leaf node. We say a pruned tree $\mathcal{T}^{(p)}$ is a valid pruned tree of $\mathcal{T}$ if (1) $\mathcal{T}^{(p)}$ is a subtree of $\mathcal{T}$ containing root node 1 and (2) for any $h\neq 1$ contained in $\mathcal{T}^{(p)}$, the sibling nodes (the set of nodes that share the same immediate parent node as $h$ in $\cT$) must also be contained in $\mathcal{T}^{(p)}$. 
Specifying a pruning is equivalent to specifying the nodes that are leaves in the pruned tree. We therefore introduce the following binary variable for each node $h\in \cT$
$$
z_h=\left\{
\begin{array}{rl}
1 & \text{if node } h \text{ is a leaf in the pruned tree} ,\\
0 & \text{otherwise}.
\end{array} \right.
$$
We call the set $\{z_h, \forall h \in \cT\}$ the node variables as they are associated with each node in the tree. Consider any root-to-leaf path in a tree $\cT$, there should be exactly one node in the path that is a leaf node in the pruned tree. Let $p(h)$ denote the set of predecessor nodes, the set of nodes (including $h$) that lie on the path from the root node to $h$. 
The set of valid pruned trees can be represented as the set of node variables satisfying the following set of constraints:
$\sum_{u\in p(h)} z_u=1 \quad \forall h \in \tilde{\mathcal{T}}$.
Given a valid pruning for a tree, we now seek to parameterize the error of the pruning.

\noindent
{\bf Pruning error: } As in most supervised empirical risk minimization problems, we aim to minimize the error on training data as a surrogate to minimizing the expected error. In a decision tree $\cT$, each node $h$ is associated with a predicted label corresponding to the majority label among the training examples that fall into the node $h$. Let $S_h$ denote the subset of examples in $S$ routed to or through node $h$ on $\mathcal{T}$ and let $\text{Pred}_h$ denote the predicted label at $h$. The number of misclassified examples at $h$ is therefore $e_h=\sum_{i\in S_h} \indicator{y^{(i)}\neq \text{Pred}_h}$. We can thus estimate the error of tree $\cT$ in terms of the number of misclassified examples in the leaf nodes: $\frac{1}{N}\sum_{h\in \tilde{\cT}}e_h$, where $N=|S|$ is the total number of examples. 

Our goal is to minimize the expected test error of the trees in the random forest, which we empirically approximate based on the aggregated probability distribution in Step~\eqref{algo:predictionRule} of Algorithm~\ref{algo:BudgetPrune} with $\frac{1}{TN}\sum_{t=1}^{T}\sum_{h\in \tilde{\cT_t}}e_h$.
We can express this error in terms of the node variables:
$\frac{1}{TN}\sum_{t=1}^{T}\sum_{h\in \cT_t}e_h z_h$.

\noindent
{\bf Pruning cost:} 
Assume the acquisition cost for the $K$ features, $\{c_k:k=1,\dots,K\}$, are given. The feature acquisition cost incurred by an example is the sum of the acquisition costs of unique features acquired in the process of running the example through the forest. This cost structure arises due to the assumption that an acquired feature is cached and subsequent usage by the same example incurs no additional cost.
Formally, the feature cost of classifying an example $i$ on the ensemble $\cT_{[T]}$ is given by 
$C_{\text{feature}}(\cT_{[T]},\bx^{(i)}) =\sum_{k=1}^{K}c_k w_{k,i}$, where the binary variables $w_{k,i}$ serve as the indicators:
$$
w_{k,i}=\left\{\begin{array}{rl}
1 & \text{ if feature } k \text{ is used by }\bx^{(i)} \text{ in any } \cT_t, t=1,\dots,T\\
0 & \text{ otherwise}.
\end{array} \right.
$$
The expected feature cost of a test example can be approximated as $\frac{1}{N}\sum_{i=1}^{N}\sum_{k=1}^{K}c_k w_{k,i}$. 

In some scenarios, it is useful to account for computation cost along with feature acquisition cost during prediction-time. In an ensemble, this corresponds to the expected number of Boolean operations required running a test through the trees, which is equal to the expected depth of the trees. This can be modeled as $\frac{1}{N}\sum_{t=1}^{T}\sum_{h\in \cT_t} |S_h| d_h z_h$, where $d_h$ is the depth of node $h$.


\noindent
{\bf Putting it together:} Having modeled the pruning constraints, prediction performance and costs, 
we formulate the problem of pruning using the relationship between the node variables $z_h$'s and feature usage variables $w_{k,i}$'s. 
Given a tree $\cT$, feature $k$, and example $\bx^{(i)}$, let $u_{k,i}$ be the first node associated with feature $k$ on the root-to-leaf path the example follows in $\cT$. Feature $k$ is used by  $\bx^{(i)}$ if and only if none of the nodes between the root and $u_{k,i}$ is a leaf. We represent this by the constraint
$w_{k,i}+\sum_{h\in p(u_{k,i})} z_h = 1$
for every feature $k$ used by example $x^{(i)}$ in $\cT$. Recall $w_{k,i}$ indicates whether or not feature $k$ is used by example $i$ and $p(u_{k,i})$ denotes the set of predecessor nodes of $u_{k,i}$. Intuitively, this constraint says that either the tree is pruned along the path followed by example $i$ before feature $k$ is acquired, in which case $z_h=1$ for some node $h\in p(u_{k,i})$ and $w_{k,i}=0$; or $w_{k,i}=1$, indicating that feature $k$ is acquired for example $i$. We extend the notations to ensemble pruning with tree index $t$: $z^{(t)}_h$ indicates whether node $h$ in $\cT_t$ is a leaf after pruning; $w^{(t)}_{k,i}$ indicates whether feature $k$ is used by the $i^{\mbox{th}}$ example in $\cT_t$; $w_{k,i}$ indicates whether feature $k$ is used by the $i^{\mbox{th}}$ example in any of the $T$ trees $\cT_1,\dots,\cT_T$; $u_{t,k,i}$ is the first node associated with feature $k$ on the root-to-leaf path the example follows in $\cT_t$; $K_{t,i}$ denotes the set of features the $i^{\mbox{th}}$ example uses on tree $\cT_t$. We arrive at the following integer program. 
\begin{equation*}
\hspace{-.5cm}\begin{array}{rlll}
\displaystyle \min_{\substack{z^{(t)}_h, w^{(t)}_{k,i}, w_{k,i} \in \{0,1\}}} &  \multicolumn{2}{l}{\overbrace{\frac{1}{NT}\displaystyle  \sum_{t=1}^{T}\sum_{h\in \mathcal{T}_t} e^{(t)}_h z^{(t)}_h}^{\mathclap{\text{error}}} +\lambda \left( \overbrace{\sum_{k=1}^{K}c_k(\frac{1}{N}\sum_{i=1}^{N}w_{k,i})}^{\mathclap{\text{feature acquisition cost}}} + \overbrace{\displaystyle \frac{1}{N}\sum_{t=1}^{T}\sum_{h\in \cT_t} |S_h| d_h z_h}^{\mathclap{\text{computational cost}}} \right) } \quad \textbf{(IP)}\\
\textrm{s.t.} &  \sum_{u\in p(h)} z^{(t)}_u=1, & \forall h \in \tilde{\mathcal{T}}_t, \forall t \in [T],  \qquad \quad \text{(feasible prunings)} \\
&  w^{(t)}_{k,i}+ \sum_{h\in p(u_{t,k,i})} z^{(t)}_h=1 , & \forall k\in K_{t,i},\forall i\in S, \forall t \in [T],  \text{ (feature usage/ tree)}\\
& w^{(t)}_{k,i} \leq w_{k,i}, & \forall k\in [K], \forall i\in S, \forall t\in [T].   \text{ (global feature usage)}
\end{array}
\end{equation*}


\noindent
{\bf Totally Unimodular constraints: } Even though integer programs are NP-hard to solve in general, we show that {\bf (IP)} can be solved exactly by solving its LP relaxation. We prove this in two steps: first, we examine the special structure of the equality constraints; then we examine the inequality constraint that couples the trees. Recall that a network matrix is one with each column having exactly one element equal to 1, one element equal to -1 and the remaining elements being 0. A network matrix defines a directed graph with the nodes in the rows and arcs in the columns. We have the following lemma.
\begin{lemma}\label{lemma1}
The equality constraints in {\bf (IP)} can be turned into an equivalent network matrix form for each tree.
\end{lemma}

\begin{proof}
We observe the first constraint $\sum_{u\in p(h)} z^{(t)}_u=1$ requires the sum of the node variables along a path to be 1. The second constraints $w^{(t)}_{k,i}+ \sum_{h\in p(u_{t,k,i})} z^{(t)}_h=1$ has a similar sum except the variable $w^{(t)}_{k,i}$. Imagine $w^{(t)}_{k,i}$ as yet another node variable for a fictitious child node of $u_{t,k,i}$ and the two equations are essentially equivalent. The rest of proof follows directly from the construction in Proposition 3 of \cite{OptimalPruning2009}.
\end{proof}
Figure \ref{fig:trees} illustrates such a construction. The nodes are numbered 1 to 5. The subscripts at node 1 and 3 are the feature index used in the nodes. Since the equality constraints in {\bf (IP)} can be separated based on the trees, we consider only one tree and one example being routed to node 4 on the tree for simplicity. The equality constraints can be organized in the matrix form as shown in the middle of Figure \ref{fig:trees}. Through row operations, the constraint matrix can be transformed to an equivalent network matrix. Such transformation always works as long as the leaf nodes are arranged in a pre-order manner. Next, we deal with the inequality constraints and obtain our main result. 

\tikzset{every tree node/.style={minimum width=1em,draw,circle},
         blank/.style={draw=none},
         edge from parent/.style=
         {draw, edge from parent path={(\tikzparentnode) -- (\tikzchildnode)}},
         level distance=1cm}

\tikzset{every tree node/.style={minimum width=1em,draw,circle},
         blank/.style={draw=none},
         edge from parent/.style=
         {draw, edge from parent path={(\tikzparentnode) -- (\tikzchildnode)}},
         level distance=1cm}
\renewcommand{\kbldelim}{(}
\renewcommand{\kbrdelim}{)}
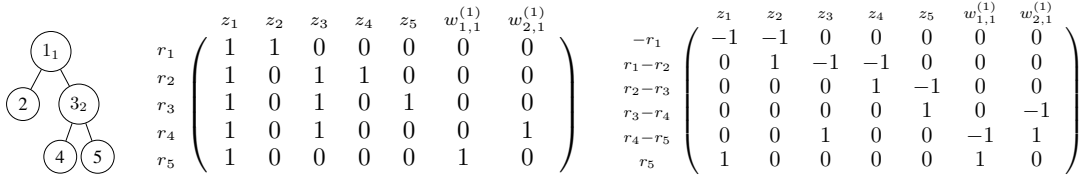
\begin{figure}
\subcaptionbox*{}{
\begin{tikzpicture}[scale=.7]
\Tree
[.\ensuremath{1_1} 
    [.2 ]    
    [.\ensuremath{3_2} 
    	[.4 ]
    	[.5 ]
    ]
]
\end{tikzpicture}}
\hspace{-1.5cm}
\subcaptionbox*{}{
\resizebox{!}{.1\hsize}{
$$
\kbordermatrix{
    & z_1 & z_2 & z_3 & z_4 & z_5 & w^{(1)}_{1,1} & w^{(1)}_{2,1}\\
    r_1 & 1 & 1 & 0 & 0 & 0 & 0 & 0 \\
    r_2 & 1 & 0 & 1 & 1 & 0 & 0 & 0 \\
    r_3 & 1 & 0 & 1 & 0 & 1 & 0 & 0 \\
    r_4 & 1 & 0 & 1 & 0 & 0 & 0 & 1 \\
    r_5 & 1 & 0 & 0 & 0 & 0 & 1 & 0
  }
$$}
}
\hspace{-1.5cm}
\subcaptionbox*{}{
\resizebox{!}{.1\hsize}{
$$
\kbordermatrix{
    & z_1 & z_2 & z_3 & z_4 & z_5 & w^{(1)}_{1,1} & w^{(1)}_{2,1}\\
    -r_1 & -1 & -1 & 0 & 0 & 0 & 0 & 0 \\
    r_1-r_2 & 0 & 1 & -1 & -1 & 0 & 0 & 0 \\
    r_2-r_3 & 0 & 0 & 0 & 1 & -1 & 0 & 0 \\
    r_3-r_4 & 0 & 0 & 0 & 0 & 1 & 0 & -1 \\
    r_4-r_5 & 0 & 0 & 1 & 0 & 0 & -1 & 1 \\
    r_5 & 1 & 0 & 0 & 0 & 0 & 1 & 0
  }
$$}
}
\vspace{-.5cm}
\caption{A decision tree example with node numbers and associated feature in subscripts together with the constraint matrix and its equivalent network matrix form.}\label{fig:trees}
\vspace{-.4cm}
\end{figure}

\begin{theorem}\label{theorem}
The LP relaxation of {\bf (IP)}, where the 0-1 integer constraints are relaxed to interval constraints $[0,1]$ for all integer variables, has integral optimal solutions.
\end{theorem}
Due to space limit the proof can be found in the Appendix. The main idea is to show the constraints are still \emph{totally unimodular} even after adding the coupling constraints and the LP relaxed polyhedron has only integral extreme points \cite{nemhauser1978analysis}. 
As a result, solving the LP relaxation results in the optimal solution to the integer program {\bf (IP)}, allowing for polynomial time optimization. \footnote{The nice result of totally unimodular constraints is due to our specific formulation. See Appendix for an alternative formulation that does not have such a property.}
\begin{algorithm}[hb]
\caption{{\textbf{\textsc{BudgetPrune}}}}\label{algo:BudgetPrune}
\begin{algorithmic}[1]
\Algphase{During Training: input - ensemble($\cT_1,\dots,\cT_T$), training/validation data with labels, $\lambda$}
\State initialize dual variables $\beta_{k,i}^{(t)} \leftarrow 0$.
\State update $z^{(t)}_h,w^{(t)}_{k,i}$ for each tree $t$ (shortest-path algo). $w_{k,i}=0$ if $\mu_{k,i}>0$, $w_{k,i}=1$ if $\mu_{k,i}<0$.
\State $\beta_{k,i}^{(t)} \leftarrow [\beta_{k,i}^{(t)}+ \gamma (w_{k,i}^{(t)}-w_{k,i})]_+$ for step size $\gamma$, where $[\cdot]_+=\max\{0,\cdot\}$.
\State go to Step 2 until duality gap is small enough. 
\Algphase{During Prediction: input - test exmaple $\mathbf{x}$}
\State Run $\mathbf{x}$ on each tree to leaf, obtain the probability distribution over label classes $\mathbf{p}_t$ at leaf.
\State Aggregate $\mathbf{p}=\frac{1}{T}\sum_{t=1}^{T}\mathbf{p}_t$. Predict the class with the highest probability in $\mathbf{p}$. \label{algo:predictionRule}
\end{algorithmic}
\end{algorithm}
\section{A Primal-Dual Algorithm}
Even though we can solve {\bf (IP)} via its LP relaxation, the resulting LP can be too large in practical applications for any general-purpose LP solver. In particular, the number of variables and constraints is roughly $O(T \times |\cT_{\text{max}}|+N \times T \times K_{\text{max}})$, where $T$ is the number of trees; $|\cT_{\text{max}}|$ is the maximum number of nodes in a tree; $N$ is the number of examples; $K_{\text{max}}$ is the maximum number of features an example uses in a tree. The runtime of the LP thus scales $O(T^{3})$ with the number of trees in the ensemble, limiting the application to only small ensembles. In this section we propose a primal-dual approach that effectively decomposes the optimization into many sub-problems. Each sub-problem corresponds to a tree in the ensemble and can be solved efficiently as a shortest path problem.  The runtime per iteration is $O(\frac{T}{p}(|\cT_{\text{max}}|+N \times K_{\text{max}})\log(|\cT_{\text{max}}|+N \times K_{\text{max}}))$, where $p$ is the number of processors. We can thus massively parallelize the optimization and scale to much larger ensembles as the runtime depends only linearly on $\frac{T}{p}$.
To this end, we assign dual variables $\beta_{k,i}^{(t)}$ for the inequality constraints $ w^{(t)}_{k,i} \leq w_{k,i}$ and derive the dual problem. 
\begin{equation*}
\hspace{-.3cm}\begin{array}{rlll}
\displaystyle \max_{\beta_{k,i}^{(t)}\geq 0} \min_{\substack{z^{(t)}_h\in [0,1] \\ w^{(t)}_{k,i}\in [0,1] \\w_{k,i}\in [0,1]}} & \multicolumn{2}{l}{\frac{1}{NT}\displaystyle  \sum_{t=1}^{T}\sum_{h\in \mathcal{T}_t} \hat{e}^{(t)}_h z^{(t)}_h +\lambda \left( \sum_{k=1}^{K}c_k(\frac{1}{N}\sum_{i=1}^{N}w_{k,i}) \right) + \sum_{t=1}^T \sum_{i=1}^N \sum_{k\in K_{t,i}} \beta_{k,i}^{(t)}(w_{k,i}^{(t)}-w_{k,i})}\\
\textrm{s.t.} &  \displaystyle \sum_{u\in p(h)} z^{(t)}_u=1, & \forall h \in \tilde{\mathcal{T}}_t, \forall t \in [T], \\
&  w^{(t)}_{k,i}+ \displaystyle \sum_{h\in p(u_{t,k,i})} z^{(t)}_h=1 , & \forall k\in K_{t,i},\forall i\in S, \forall t \in [T], \\
\end{array}
\end{equation*}
where for simplicity we have combined coefficients of $z_h^{(t)}$ in the objective of {\bf (IP)} to $\hat{e}_h^{(t)}$.
The primal-dual algorithm is summarized in Algorithm \ref{algo:BudgetPrune}. It alternates between updating the primal and the dual variables. The key is to observe that given dual variables, the primal problem (inner minimization) can be decomposed for each tree in the ensemble and solved in parallel as shortest path problems due to Lemma \ref{lemma1}. (See also Appendix). The primal variables $w_{k,i}$ can be solved in closed form: simply compute $\mu_{k,i}=\lambda c_k / N-\sum_{t\in T_{k,i}} \beta_{k,i}^{(t)}$, where $T_{k,i}$ is the set of trees in which example $i$ encounters feature $k$. So $w_{k,i}$ should be set to 0 if $\mu_{k,i}>0$ and $w_{k,i}=1$ if $\mu_{k,i}<0$.

Note that our prediction rule aggregates the leaf distributions from all trees instead of just their predicted labels. In the case where the leaves are pure (each leaf contains only one class of examples), this prediction rule coincides with the majority vote rule commonly used in random forests. Whenever the leaves contain mixed classes, this rule takes into account the prediction confidence of each tree in contrast to majority voting. Empirically, this rule consistently gives lower prediction error than majority voting with pruned trees.

\section{Experiments}
We test our pruning algorithm \textsc{BudgetPrune} on four benchmark datasets used for prediction-time budget algorithms. The first two datasets have unknown feature acquisition costs so we assign costs to be 1 for all features; the aim is to show that \textsc{BudgetPrune} successfully selects a sparse subset of features on average to classify each example with high accuracy. \footnote{In contrast to traditional sparse feature selection, our algorithm allows adaptivity, meaning different examples use different subsets of features.} The last two datasets have real feature acquisition costs measured in terms of CPU time. \textsc{BudgetPrune} achieves high prediction accuracy spending much less CPU time in feature acquisition. 

For each dataset we first train a RF and apply \textsc{BudgetPrune} on it using different $\lambda$'s to obtain various points on the accuracy-cost tradeoff curve.  We use in-bag data to estimate error probability at each node and the validation data for the feature cost variables $w_{k,i}$'s. We implement \textsc{BudgetPrune} using CPLEX \cite{cplex} network flow solver for the primal update step. The running time is significantly reduced (from hours down to minutes) compared to directly solving the LP relaxation of {\bf (IP)} using standard solvers such as Gurobi \cite{gurobi}. Futhermore, the standard solvers simply break trying to solve the larger experiments whereas \textsc{BudgetPrune} handles them with ease. 
We run the experiments for 10 times and report the means and standard deviations.
 
\noindent
{\bf Competing methods:} We compare against four other approaches. ({\bf i}) \textsc{BudgetRF}\cite{icml2015_nan15}: the recursive node splitting process for each tree is stopped as soon as node impurity (entropy or Pairs) falls below a threshold. The threshold is a measure of impurity tolerated in the leaf nodes. This can be considered as a naive pruning method as it reduces feature acquisition cost while maintaining low impurity in the leaves. 
\begin{wrapfigure}{r}{0.8\textwidth}
\vspace{-0.35cm}
\centering
\subcaptionbox{MiniBooNE}{\includegraphics[width=.49\linewidth]{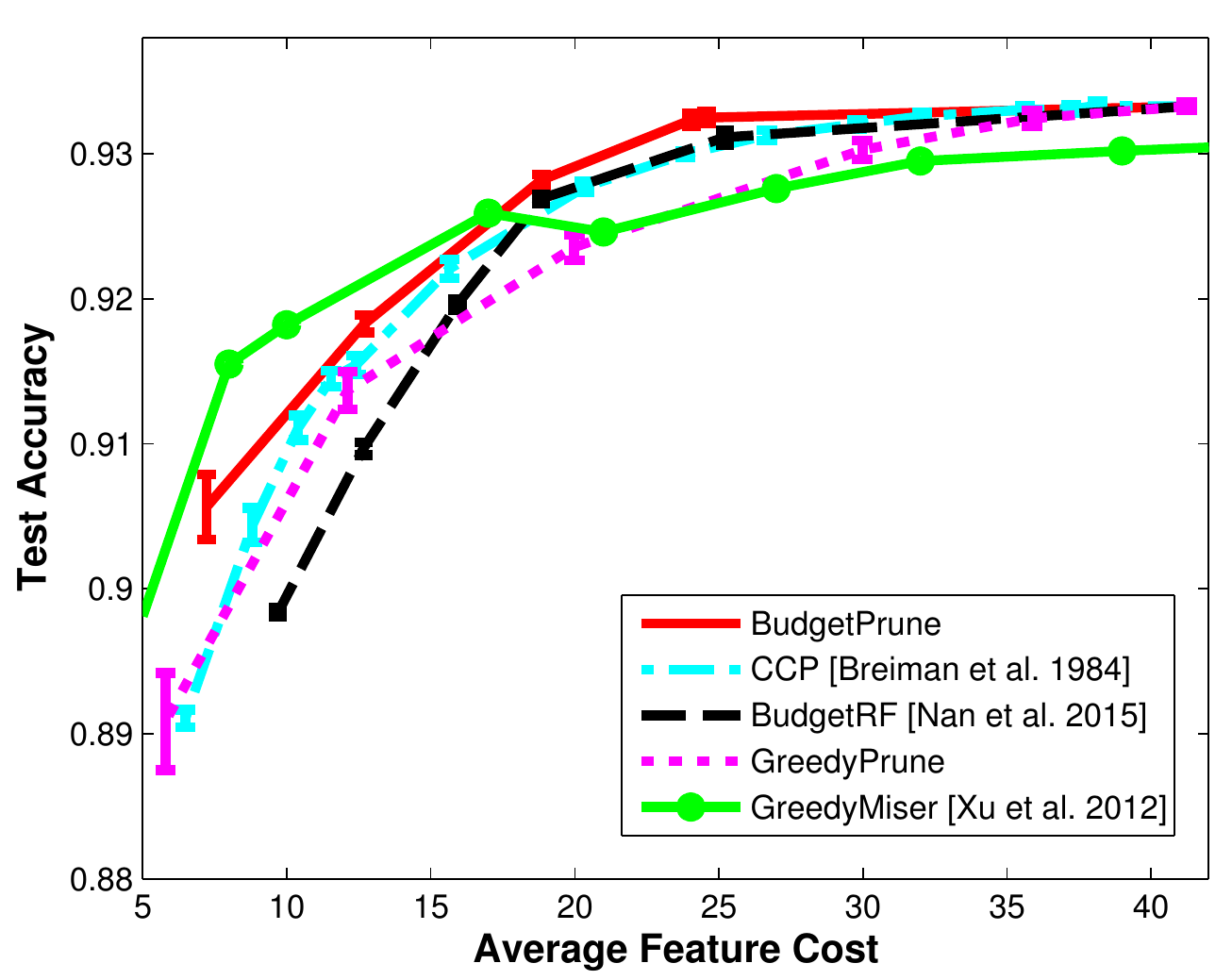}}
\subcaptionbox{Forest Covertype}{\includegraphics[width=.49\linewidth]{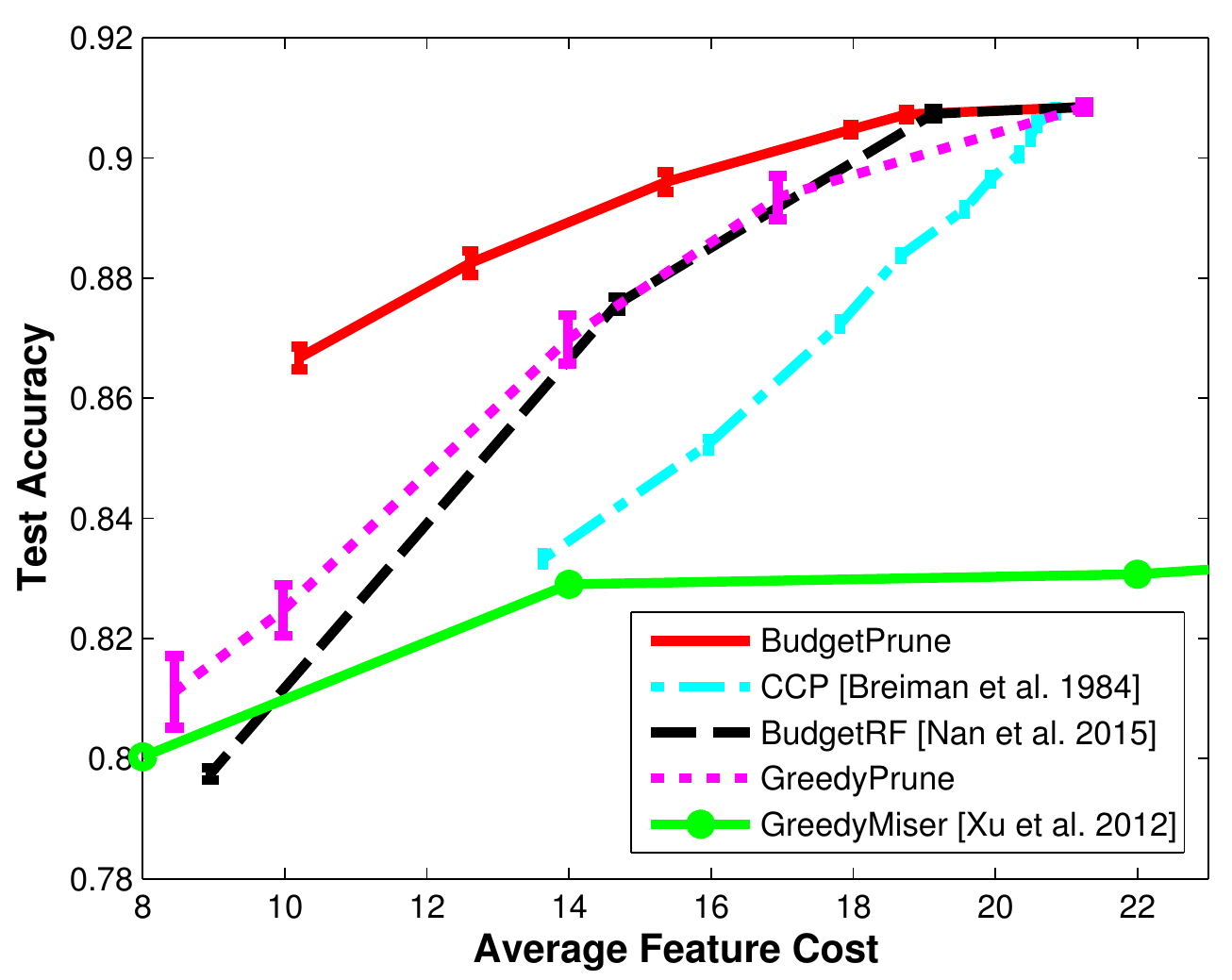}} \\
\vspace{-.1cm}
\subcaptionbox{Yahoo! Rank}{\includegraphics[width=.49\linewidth]{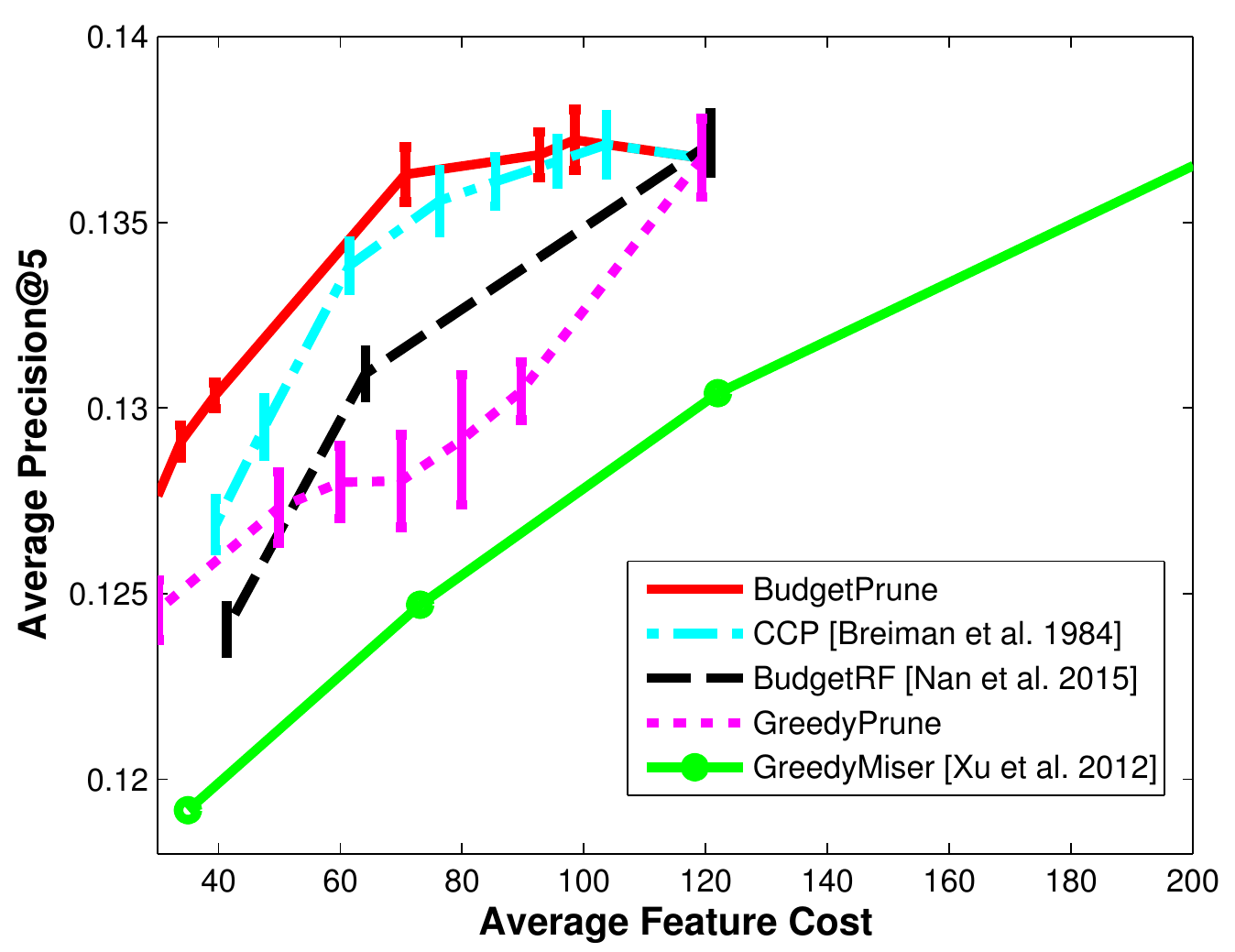}} 
\subcaptionbox{Scene15}{\includegraphics[width=.49\linewidth]{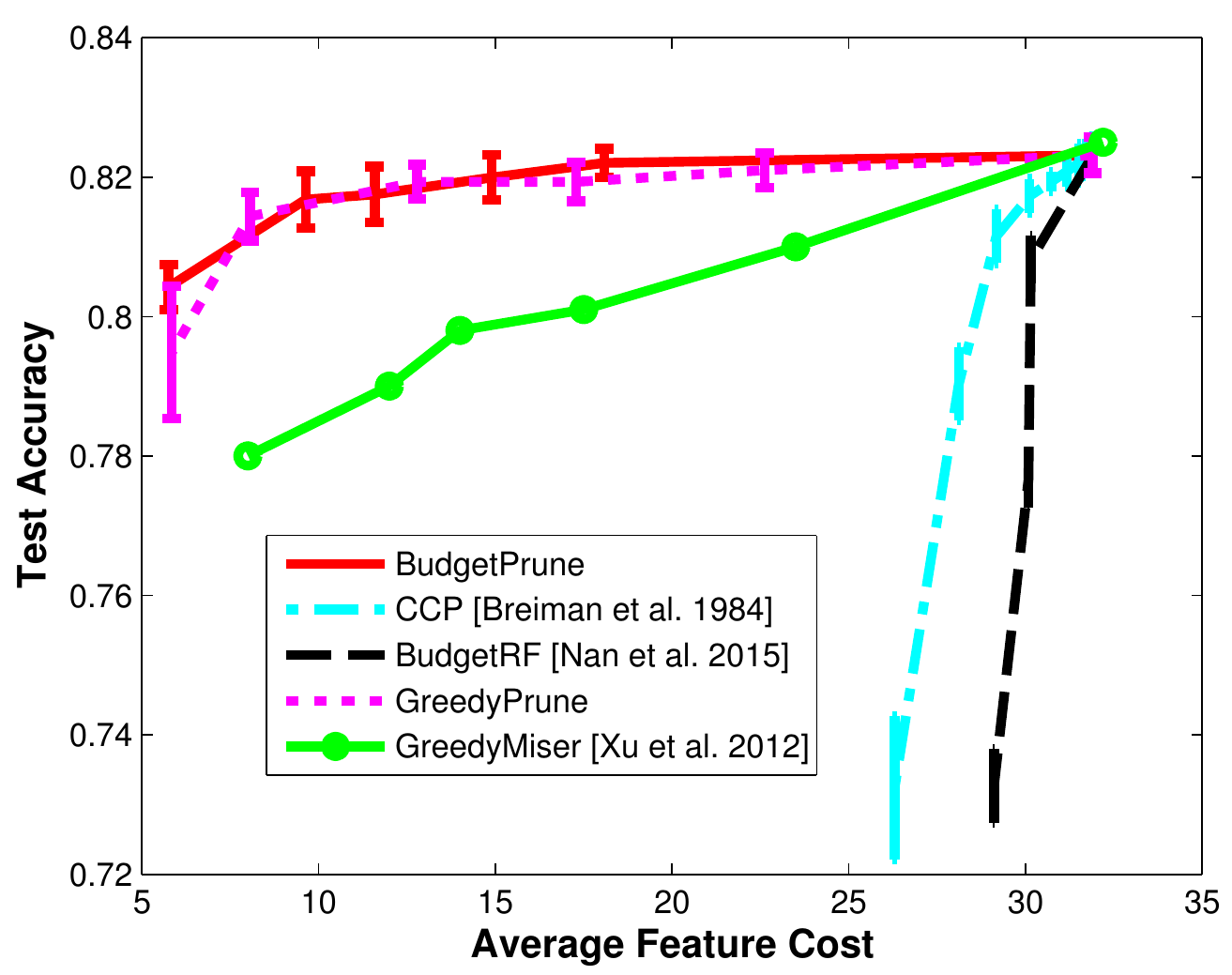}}
\vspace{-.25cm}
\caption{\footnotesize Comparison of \textsc{BudgetPrune} against CCP, \textsc{BudgetRF} with early stopping, \textsc{GreedyPrune} and \textsc{GreedyMiser}  on 4 real world datasets. \textsc{BudgetPrune} (red) outperforms competing state-of-art methods. \textsc{GreedyMiser} dominates ASTC~\cite{ASTC_AAAI14}, CSTC~\cite{xu2013cost} and DAG~\cite{NIPS2015_5982} significantly on all datasets. We omit them in the plots to clearly depict the differences between competing methods.}
\label{fig:experiments}
\vspace{-0.1in}
\end{wrapfigure}
({\bf ii}) Cost-Complexity Pruning (CCP) \cite{breiman1984classification}: it iteratively prunes subtrees such that the resulting tree has low error and small size. We perform CCP on individual trees to different levels to obtain various points on the accuracy-cost tradeoff curve. CCP does not take into account feature costs. ({\bf iii}) \textsc{GreedyPrune}: is a greedy global feature pruning strategy that we propose; at each iteration it attempts to remove all nodes corresponding to one feature from the RF such that the resulting pruned RF has the lowest training error and average feature cost. The process terminates in at most K iterations, where K is the number of features. The idea is to reduce feature costs by successively removing features that result in large cost reduction yet small accuracy loss. We also compare against the state-of-the-art methods in budgeted learning ({\bf iv}) \textsc{GreedyMiser} \cite{DBLP:conf/icml/XuWC12}: it is a modification of gradient boosted regression tree \cite{Friedman00greedyfunction} to incorporate feature cost. Specifically, each weak learner (a low-depth decision tree) is built to minimize squared loss with respect to current gradient at the training examples plus feature acquisition cost. To build each weak learner the feature costs are set to zero for those features already used in previous weak learners. Other prediction-time budget algorithms such as ASTC \cite{ASTC_AAAI14}, CSTC \cite{xu2013cost} and cost-weighted $l$-1 classifiers are shown to perform strictly worse than \textsc{GreedyMiser} by a significant amount \cite{ASTC_AAAI14,icml2015_nan15} so we omit them in our plots. Since only the feature acquisition costs are standardized, for fair comparison we do not include the computation cost term in the objective of {\bf (IP)} and focus instead on feature acquisition costs.


\noindent
{\bf MiniBooNE Particle Identification and Forest Covertype Datasets:\cite{UCI_repository}} Feature costs are uniform in both datasets. Our base RF consists of 40 trees using entropy split criteria and choosing from the full set of features at each split. 
As shown in (a) and (b) of Figure \ref{fig:experiments}, \textsc{BudgetPrune} (in red) achieves the best accuracy-cost tradeoff. The advantage of \textsc{BudgetPrune} is particularly large in (b). \textsc{GreedyMiser} has lower accuracy in the high budget region compared to \textsc{BudgetPrune} in (a) and significantly lower accuracy in (b). The gap between \textsc{BudgetPrune} and other pruning methods is small in (a) but much larger in (b). This indicates large gains from globally encouraging feature sharing in the case of (b) compared to (a). In both datasets, \textsc{BudgetPrune} successfully prunes away large number of features while maintaining high accuracy. For example in (a), using only 18 unique features on average instead of 40, we can get essentially the same accuracy as the original RF.

\noindent
{\bf Yahoo! Learning to Rank:\cite{YahooChallenge2010}}
This ranking dataset consists of $473134$ web documents and $19944$ queries. Each example in the dataset contains features of a query-document pair together with the relevance rank of the document to the query. There are $141397/146769/184968$ examples in the training/validation/test sets. There are 519 features for each example; each feature is associated with an acquisition cost in the set $\{1,5,20,50,100,150,200\}$, which represents the units of CPU time required to extract the feature and is provided by a Yahoo! employee. The labels are binarized so that the document is either relevant or not relevant to the query. 
The task is to learn a model that takes a new query and its associated set of documents to produce an accurate ranking using as little feature cost as possible. 
As in \cite{icml2015_nan15}, we use the Average Precision@5 as the performance metric, which gives a high reward for ranking the relevant documents on top. Our base RF consists of 140 trees using cost weighted entropy split criteria as in \cite{icml2015_nan15} and choosing from a random subset of 400 features at each split. As shown in (c) of Figure \ref{fig:experiments}, \textsc{BudgetPrune} achieves similar ranking accuracy as \textsc{GreedyMiser} using only 30\% of its cost. 

\noindent
{\bf Scene15 \cite{scene15}:}
This scene recognition dataset contains 4485 images from 15 scene classes (labels). Following \cite{DBLP:conf/icml/XuWC12} we divide it into $1500/300/2685$ examples for training/validation/test sets. We use a diverse set of visual descriptors and object detectors from the Object Bank \cite{objectBank}. We treat each individual detector as an independent descriptor so we have a total of 184 visual descriptors. The acquisition costs of these visual descriptors range from 0.0374 to 9.2820. For each descriptor we train 15 one-vs-rest kernel SVMs and use the output (margins) as features. Once any feature corresponding to a visual descriptor is used for a test example, an acquisition cost of the visual descriptor is incurred and subsequent usage of features from the same group is free for the test example. 
Our base RF consists of 500 trees using entropy split criteria and choosing from a random subset of 20 features at each split. As shown in (d) of Figure \ref{fig:experiments}, \textsc{BudgetPrune} and \textsc{GreedyPrune} significantly outperform other competing methods. \textsc{BudgetPrune} has the same accuracy at the cost of 9 as at the full cost of 32. \textsc{BudgetPrune} and \textsc{GreedyPrune} perform similarly, indicating the greedy approach happen to solve the global optimization in this particular initial RF. 

\subsection{Discussion \& Concluding Comments} We have empirically evaluated several resource constrained learning algorithms including \textsc{BudgetPrune} and its variations on benchmarked datasets here and in the Appendix. We highlight key features of our approach below. \\
({\bf i}) \textsc{State-of-the-art Methods}. Recent work has established that \textsc{GreedyMiser} and \textsc{BudgetRF} are among the state-of-the-art methods dominating a number of other methods~\cite{ASTC_AAAI14,xu2013cost,NIPS2015_5982} on these benchmarked datasets. \textsc{GreedyMiser} requires building class-specific ensembles and tends to perform poorly and is increasingly difficult to tune in multi-class settings. RF, by its nature, can handle multi-class settings efficiently. On the other hand, as we described earlier, \cite{ASTC_AAAI14,NIPS2015_5982,xu2013cost} are fundamentally "tree-growing" approaches, namely they are top-down methods acquiring features sequentially based on a surrogate utility value. This is a fundamentally combinatorial problem that is known to be NP hard~\cite{DecisionTreesforEntityIdentification,xu2013cost} and thus requires a number of relaxations and heuristics with no guarantees on performance. In contrast our pruning strategy is initialized to realize good performance (RF initialization) and we are able to globally optimize cost-accuracy objective. \\
({\bf ii}) \textsc{Variations on Pruning}. By explicitly modeling feature costs, \textsc{BudgetPrune} outperforms other pruning methods such as early stopping of \textsc{BudgetRF} and CCP that do not consider costs. \textsc{GreedyPrune} performs well validating our intuition (see Table.~1) that pruning sparsely occurring feature nodes utilized by large fraction of examples can improve test-time cost-accuracy tradeoff. Nevertheless, the \textsc{BudgetPrune} outperforms \textsc{GreedyPrune}, which is indicative of the fact that apart from obvious high-budget regimes, node-pruning must account for how removal of one node may have an adverse impact on another downstream one. \\
({\bf iii}) \textsc{Sensitivity to Impurity, Feature Costs, \& other inputs}. We explore these issues in Appendix. We experiment \textsc{BudgetPrune} with different impurity functions such as entropy and Pairs \cite{icml2015_nan15} criteria. Pairs-impurity tends to build RFs with lower cost but also lower accuracy compared to entropy and so has poorer performance. We also explored how non-uniform costs can impact cost-accuracy tradeoff. An elegant approach has been suggested by \cite{benbouzid:tel-00990245}, who propose an adversarial feature cost proportional to feature utility value. We find that \textsc{BudgetPrune} is robust with such costs. Other RF parameters including number of trees and feature subset size at each split do impact cost-accuracy tradeoff in obvious ways with more trees and moderate feature subset size improving prediction accuracy while incurring higher cost.

To conclude, our proposed formulation possesses 1) elegant theoretical properties, 2) an algorithm  scalable to large problems and 3) superior empirical performance.

\paragraph{Acknowledgment} We thank Dr Kilian Weinberger for helpful discussions and Dr David Castanon for the insights on the primal dual algorithm.

\section{Appendix}
\subsection{A Naive Pruning Formulation}
The nice property of totally unimodular constraints in Theorem 3.2 is due to our specific formulation. 
Here we present an alternative integer program formulation and show its deficiency. Recall we defined the following node variables 
$$
z_h=\left\{
\begin{array}{rl}
1 & \text{if node } h \text{ is a leaf in the pruned tree} ,\\
0 & \text{otherwise}.
\end{array} \right.
$$
and indicator variables of feature usage:
$$
w_{k,i}=\left\{\begin{array}{rl}
1 & \text{ if feature } k \text{ is used by }\bx^{(i)} \text{ in any } \cT_t, t=1,\dots,T\\
0 & \text{ otherwise}.
\end{array} \right.
$$
First, note that if $z_h=1$ for some node $h$, then the examples that are routed to $h$ must have used all the features in the predecessor nodes $p(h)$, excluding $h$. We use $k\sim p(h)$ to denote feature $k$ is used in any predecessor of $h$, excluding $h$. Then for each feature $k$ and example $i$, we must have $w_{k,i}\geq z_h$ for all nodes $h$ such that $i\in S_h$ and $k\sim p(h)$. Combining these constraints with the pruning constraints we formulate pruning as a 0-1 integer program for an individual tree:
\begin{equation*}
\begin{array}{rlll}
\displaystyle \min_{\substack{z_h \in \{0,1\} \\w_{k,i}\in \{0,1\}}} & \multicolumn{2}{l}{\frac{1}{N}\displaystyle  \sum_{h\in \mathcal{N}} e_h z_h +\lambda \sum_{k=1}^{K}c_k(\frac{1}{N}\sum_{i=1}^{N}w_{k,i})} \\
\textrm{s.t.} & z_h+ \sum_{u\in p(h)} z_u=1 & \forall h \in \tilde{\mathcal{T}}, \\
& w_{k,i}\geq z_h & \forall h:i\in S_h \land k\sim p(h), \\
& & \forall k\in [K], \forall i\in S.   \\
\end{array}
\label{eq:IP1}
\end{equation*}
To solve the integer program, a common heuristic is to solve its linear program relaxation. Unfortunately, the constraint set in the above formulation has fractional extreme points, leading to possibly fractional solutions to the relaxed problem. It is not clear how to perform rounding to obtain good prunings. Consider the first tree in Figure \ref{fig:trees}. Feature 1 is used at the root node and feature 2 is used at node 3. There are 7 variables (assuming there is only one example and it goes to leaf 4): $z_1,z_2,z_3,z_4,z_5,w_{1,1},w_{2,1}$. 
The LP relaxed constraints are:
\begin{align*}
& z_1+z_3+z_4=1 , z_1+z_3+z_5=1 , z_1+z_2=1, \\
& w_{1,1}\geq z_4, w_{1,1}\geq z_3, w_{2,1}\geq z_4, 0\leq z\leq 1.
\end{align*}
The following is a basic feasible solution:
\begin{equation*}
z_1=0,  z_2=1 , z_3=z_4=z_5=0.5, w_{1,1}=w_{2,1}=0.5,
\end{equation*}
because the following set of 7 constraints are active:
\begin{align*}
& z_1+z_3+z_4=1, z_1+z_3+z_5=1, \\
& w_{1,1}\geq z_4, w_{1,1}\geq z_3,w_{2,1}\geq z_4, z_1=0,z_2=1.
\end{align*}
Even if we were to interpret the fractional solution of $z_h$ as probabilities of $h$ being a leaf node, we see an issue with this formulation: the example has $0.5$ probability of stopping at node 3 or 4 ($z_3=z_4=0.5$). In both cases, feature 1 at the root node has to be used, however $w_{1,1}=0.5$ indicates that it is only being used half of the times. This solution is not a feasible pruning and fails to capture the cost of the pruning.

Attempting to use an LP relaxation of this formulation fails to capture the desired behavior of the integer program. In the main paper we propose a better integer program formulation and show that solving the LP relaxation yields the optimal solution to the integer program.

\subsection{Transformation to Network Matrices and Shortest Path Problems}
To illustrate the transformation to network matrix in Lemma 3.1, we provide the following illustration in Figure \ref{fig:trees}. Note in the main paper we have shown the example of the first tree. For simplicity we consider only one example being routed to nodes 4 and 11 respectively on the two trees. The equality constraints in (IP2) can be separated based on the trees and put in matrix form:
\renewcommand{\kbldelim}{(}
\renewcommand{\kbrdelim}{)}
\[
\kbordermatrix{
    & z_1 & z_2 & z_3 & z_4 & z_5 & w^{(1)}_{1,1} & w^{(1)}_{2,1}\\
    r_1 & 1 & 1 & 0 & 0 & 0 & 0 & 0 \\
    r_2 & 1 & 0 & 1 & 1 & 0 & 0 & 0 \\
    r_3 & 1 & 0 & 1 & 0 & 1 & 0 & 0 \\
    r_4 & 1 & 0 & 1 & 0 & 0 & 0 & 1 \\
    r_5 & 1 & 0 & 0 & 0 & 0 & 1 & 0
  },
\]
for tree 1 and 
\[
\kbordermatrix{
    & z_6 & z_7 & z_8 & z_9 & z_{10} & z_{11} & z_{12} & w^{(2)}_{2,1} & w^{(2)}_{3,1}\\
    r_1 & 1 & 1 & 1 & 0 & 0 & 0 & 0 & 0 & 0\\
    r_2 & 1 & 1 & 0 & 1 & 0 & 0 & 0 & 0 & 0\\
    r_3 & 1 & 0 & 0 & 0 & 1 & 1 & 0 & 0 & 0\\
    r_4 & 1 & 0 & 0 & 0 & 1 & 0 & 1 & 0 & 0\\
    r_5 & 1 & 0 & 0 & 0 & 1 & 0 & 0 & 0 & 1\\
    r_6 & 1 & 0 & 0 & 0 & 0 & 0 & 0 & 1 & 0
  },
\]
for tree 2. Through row operations they can be turned into network matrices, where there is exactly two non-zeros in each column, a 1 and a $-1$.
\[
\kbordermatrix{
    & z_1 & z_2 & z_3 & z_4 & z_5 & w^{(1)}_{1,1} & w^{(1)}_{2,1}\\
    -r_1 & -1 & -1 & 0 & 0 & 0 & 0 & 0 \\
    r_1-r_2 & 0 & 1 & -1 & -1 & 0 & 0 & 0 \\
    r_2-r_3 & 0 & 0 & 0 & 1 & -1 & 0 & 0 \\
    r_3-r_4 & 0 & 0 & 0 & 0 & 1 & 0 & -1 \\
    r_4-r_5 & 0 & 0 & 1 & 0 & 0 & -1 & 1 \\
    r_5 & 1 & 0 & 0 & 0 & 0 & 1 & 0
  },
\]
for tree 1 and 
\[
 \kbordermatrix{
    & z_6 & z_7 & z_8 & z_9 & z_{10} & z_{11} & z_{12} & w^{(2)}_{2,1} & w^{(2)}_{3,1}\\
    -r_1 & -1 & -1 & -1 & 0 & 0 & 0 & 0 & 0 & 0\\
    r_1-r_2 & 0 & 0 & 1 & -1 & 0 & 0 & 0 & 0 & 0\\
    r_2-r_3 & 0 & 1 & 0 & 1 & -1 & -1 & 0 & 0 & 0\\
    r_3-r_4 & 0 & 0 & 0 & 0 & 0 & 1 & -1 & 0 & 0\\
    r_4-r_5 & 0 & 0 & 0 & 0 & 0 & 0 & 1 & 0 & -1\\
    r_5-r_6 & 0 & 0 & 0 & 0 & 1 & 0 & 0 & -1 & 1\\
    r_6 & 1 & 0 & 0 & 0 & 0 & 0 & 0 & 1 & 0
  }
\]
for tree 2.
Note the above transformation to network matrices can always be done as long as the leaf nodes are arranged in a pre-order fashion. 

In the primal-dual algorithm, the inner minimization can be decomposed to shortest path problems corresponding to individual trees. Figure \ref{fig:shortestpath} illustrates such a construction based on the network matrices shown above. The nodes in the graphs correspond to rows in the network matrices and the arcs correspond to the columns, which are the primal variables $z_h, w^{(t)}_{k,i}$'s. There is a cost associated with each arc in the objective of the minimization problem. The task is to find a path from the first node (source) to the last node (sink) such that the sum of arc costs is minimized. Note each path from source to sink corresponds to a feasible pruning. For example, in (a) of Figure \ref{fig:shortestpath}, consider the path of 1-2-5-6, the active arcs are $z_2, z_3$ and $w^{(1)}_{1,1}$, Setting these variables to 1 and others to 0, we see that it corresponds to pruning Tree 1 at node 3 in Figure \ref{fig:trees}. (Note the nodes in Figure \ref{fig:shortestpath} and Figure \ref{fig:trees} are not to be confused - they do not have a relation with each other. )

\begin{figure} 
\vspace{-0.35cm}
\centering
\subcaptionbox{Tree 1}{\includegraphics[width=\linewidth,height=.3\linewidth]{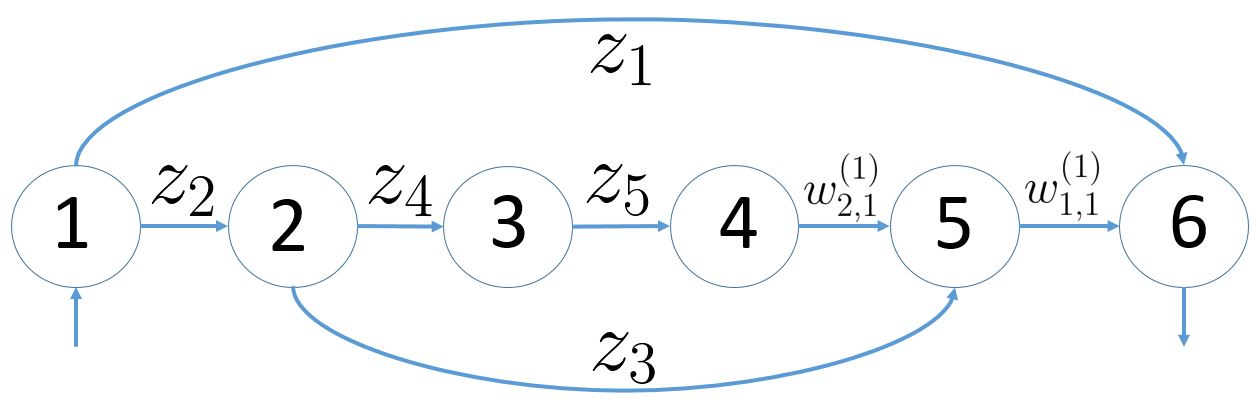}}
\subcaptionbox{Tree 2}{\includegraphics[width=\linewidth,height=.3\linewidth]{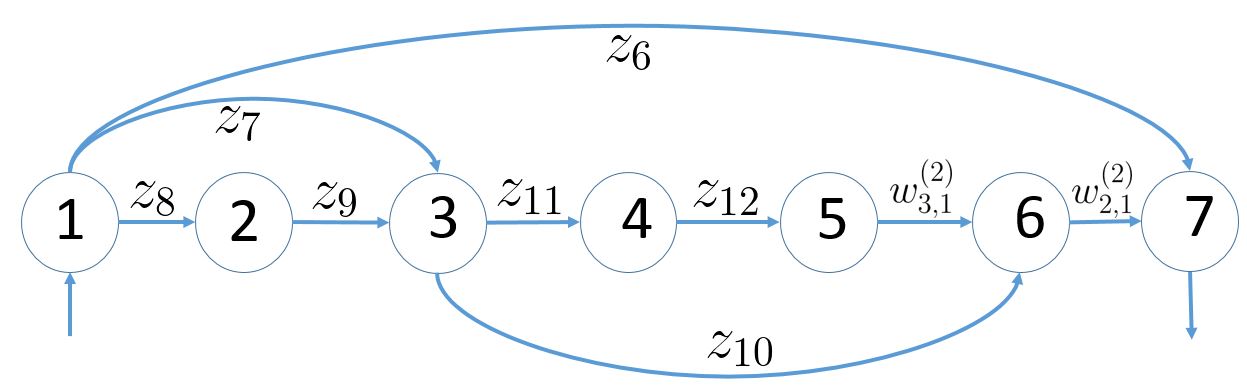}} \\
\caption{Turning pruning to equivalent shortest path problems.} \label{fig:shortestpath}
\end{figure}

\subsection{Proof of Theorem 3.2}
Denote the equality constraints of (IP) with index set $J_1$. They can be divided into each tree. 
Each constraint matrix in $J_1$ associated with a tree can be turned into a network matrix according to Lemma 3.1. Stacking these matrices leads to a larger network matrix. Denote the $w^{(t)}_{k,i}\leq w_{k,i}$ constraints with index set $J_2$. Consider the constraint matrix for $J_2$. Each $w^{(t)}_{k,i}$ only appears once in $J_2$, which means the column corresponding to $w^{(t)}_{k,i}$ has only one element equal to 1 and the rest equal to 0. If we arrange the constraints in $J_2$ such that for any given $k,i$ $w^{(t)}_{k,i}\leq w_{k,i}$ are put together for $t\in [T]$, the constraint matrix for $J_2$ has interval structure such that the non-zeros in each column appear consecutively. Finally, putting the network matrix from $J_1$ and the matrix from $J_2$ together. Assign $J_1$ and the odd rows of $J_2$ to the first partition $Q_1$ and assign the even rows of $J_2$ to the second partition $Q_2$. Note the upper bound constraints on the variables can be ignored as this is an minimization problem. We conclude that the constraint matrix of (IP) is totally unimodular according to Theorem 2.7, Part 3 of \cite{Nemhauser:1988:ICO:42805} with partition $Q_1$ and $Q_2$. By Proposition 2.1 and 2.2, Part 3 of \cite{Nemhauser:1988:ICO:42805} we can conclude the proof.

\tikzset{every tree node/.style={minimum width=1em,draw,circle},
         blank/.style={draw=none},
         edge from parent/.style=
         {draw, edge from parent path={(\tikzparentnode) -- (\tikzchildnode)}},
         level distance=1cm}
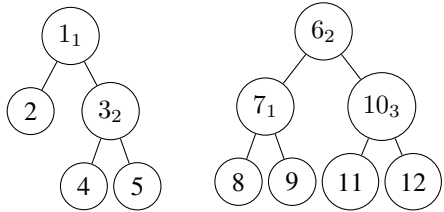
\begin{figure}
\centering
\subcaptionbox*{}{
\begin{tikzpicture}
\Tree
[.\ensuremath{1_1} 
    [.2 ]    
    [.\ensuremath{3_2} 
    	[.4 ]
    	[.5 ]
    ]
]
\end{tikzpicture}}
~
~
~
\subcaptionbox*{}{
\begin{tikzpicture}
\Tree
[.\ensuremath{6_2} [.\ensuremath{7_1} [.8 ] [.9 ] ]  [.\ensuremath{10_3} [.11 ] [.12 ] ] ]
\end{tikzpicture}}
\caption{An ensemble of two decision trees with node numbers and associated feature in subscripts}\label{fig:trees2}
\end{figure}

\subsection{Additional Details of Experiments}
In this section we provide additional details of the experiment setup and explore how some parameter choices may affect \textsc{BudgetPrune}.

\paragraph{Additional details of datasets}
The MiniBooNE data set is a binary classification task to distinguish electron neutrinos from muon neutrinos. There are $45523/19510/65031$ examples in training/validation/test sets. Each example has 50 features, each with unit cost. 
The Forest data set contains cartographic variables to predict 7 forest cover types. There are $36603/15688/58101$ examples in training/validation/test sets. Each example has 54 features, each with unit cost. We use 1000 trees for \textsc{GreedyMiser} and search over learning rates in $[10^{-5}, 10^2]$ for MiniBooNE and Forest.
The Yahoo and Scene15 datasets have actual feature acquisition costs in terms of CPU time. We use 3000 trees for \textsc{GreedyMiser} and search over learning rates in $[10^{-5}, 1]$. We use the multi-class logistic loss for Scene15 and the squared loss for other datasets in \textsc{GreedyMiser}.
For the Scene15 dataset, we use a diverse set of visual discriptors varying in computation time: GIST, spatial HOG, Local Binary Pattern, self-similarity, texton histogram, geometric texton, geometric color and 177 object detectors from the Object Bank \cite{objectBank}. We treat each individual detector as an independent descriptor so we have 184 different visual descriptors in total. The acquisition costs of these visual descriptors range from 0.0374 to 9.2820. For each descriptor we train 15 one-vs-rest kernel SVMs and use the output (margins) as features. The best classifier based on individual descriptors achieves an accuracy of 77.8\%. Note the features are grouped based on the visual descriptors. Once any feature corresponding to a visual descriptor is used for a test example, an acquisition cost of the visual descriptor is incurred and subsequent usage of features from the same group is free for the test example. 

Next, we perform additional experiments to evaluate \textsc{BudgetPrune} with different costs, input RFs.
\paragraph{Non-uniform cost on MiniBooNE}
We observe that CCP performs similarly to \textsc{BudgetPrune} on MiniBooNE when the costs are uniform in the case of entropy splitting criteria, indicating little gain from global optimization with respect to feature usage. We suspect that uniform feature costs work in favor of CCP because there's no loss in treating each feature equally. To confirm this intuition we assign the features non-uniform costs and re-run prunings on the same RF. We first normalize the data so that the data vectors corresponding to the features have the same $l$-2 norm. 
We then train a linear SVM on it and obtain the weight vector corresponding to the learned hyperplane. We around the absolute values of the weights and make them the costs for the features. Intuitively the feature with higher weight tends to be more relevant for the classification task so we assign it a higher acquisition cost. The resulting costs lie in the range of $[1,40]$ and we normalize them so that the sum of all feature costs is 50 - the number of features. We plot \textsc{BudgetPrune} and CCP for uniform cost as well as the non-uniform cost described above in Figure \ref{fig:mbne_nonuniform}. \textsc{BudgetPrune} still achieves similar performance as uniform cost while CCP performance drops significantly with non-uniform feature cost. This shows again the importance of taking into account feature costs in the pruning process.

\begin{figure}[htbp]
\centering
\includegraphics[width=0.5\textwidth]{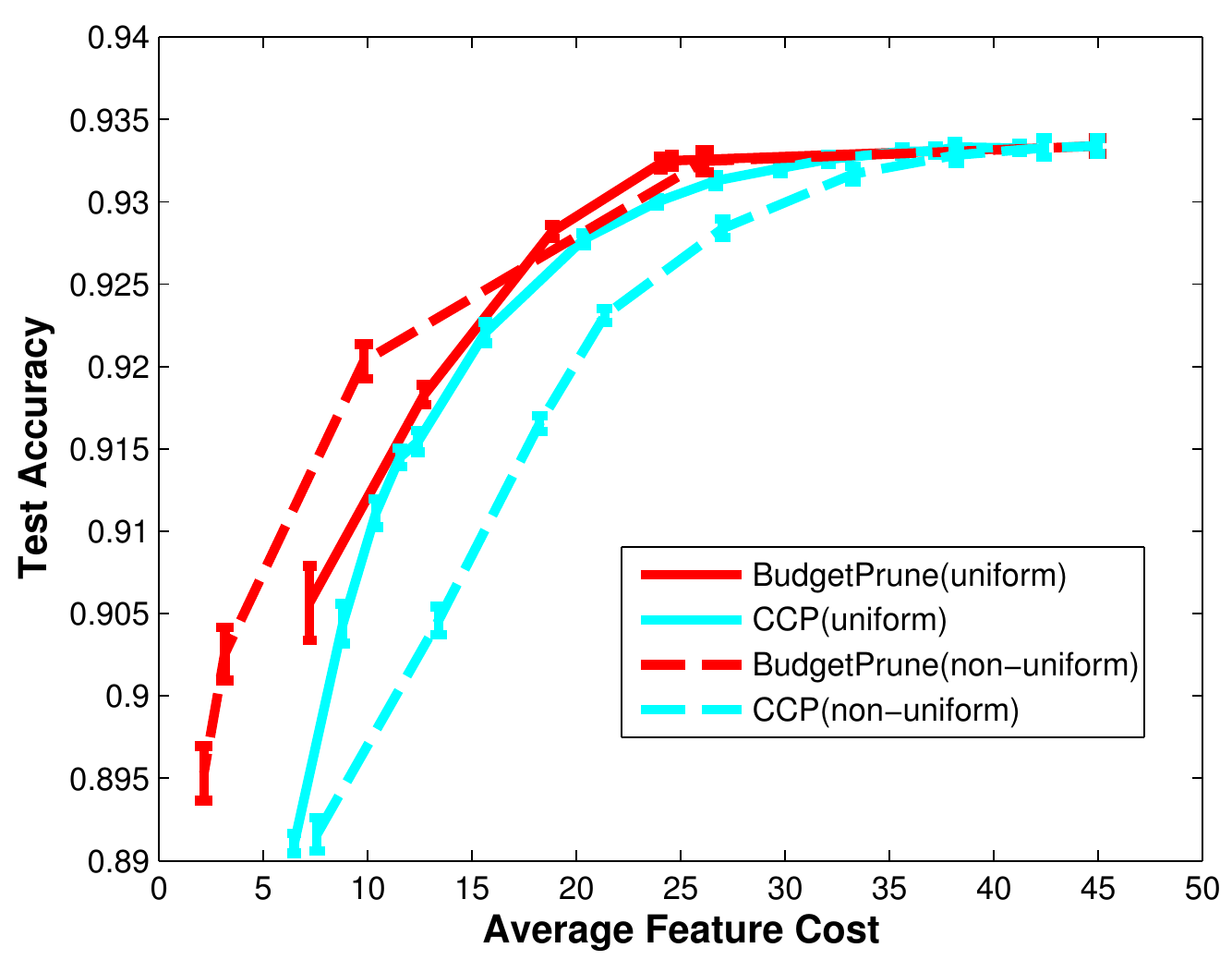}
\caption{Comparing \textsc{BudgetPrune} and CCP with uniform and non-uniform feature cost on MiniBooNE dataset. \textsc{BudgetPrune} is robust when the feature cost is non-uniform.} \label{fig:mbne_nonuniform}
\end{figure}
\paragraph{Entropy Vs Pairs}
How does \textsc{BudgetPrune} depend on the splitting criteria used in the underlying random forest? On two data sets we build RFs using the popular entropy splitting criteria and the mini-max Pairs criteria used in \cite{icml2015_nan15} and the results are shown in Figure \ref{fig:entropy_pair}. We observe that entropy splitting criteria lead to RFs with higher accuracy while the Pairs criteria lead to RFs with lower cost. This is expected as using Pairs biases to more balanced splits and thus provably low cost \cite{icml2015_nan15}. In (a) of Figure \ref{fig:entropy_pair} we observe that as more of the RF is pruned away \textsc{BudgetPrune} and CCP results for entropy and Pairs coincide. This suggests that the two criteria actually lead to similar tree structures in the initial tree-building process. However, as the trees are built deeper their structures diverge. Plot (b) in Figure \ref{fig:entropy_pair} shows that pruning based on the RFs from the Pairs criteria can achieve higher accuracy in the low cost region. But if high accuracy in the high cost region is desirable then the entropy criteria should be used. 

\begin{figure}[htbp]
\centering
\subcaptionbox*{MiniBooNE}{\includegraphics[width=.48\linewidth,height=.4\linewidth]{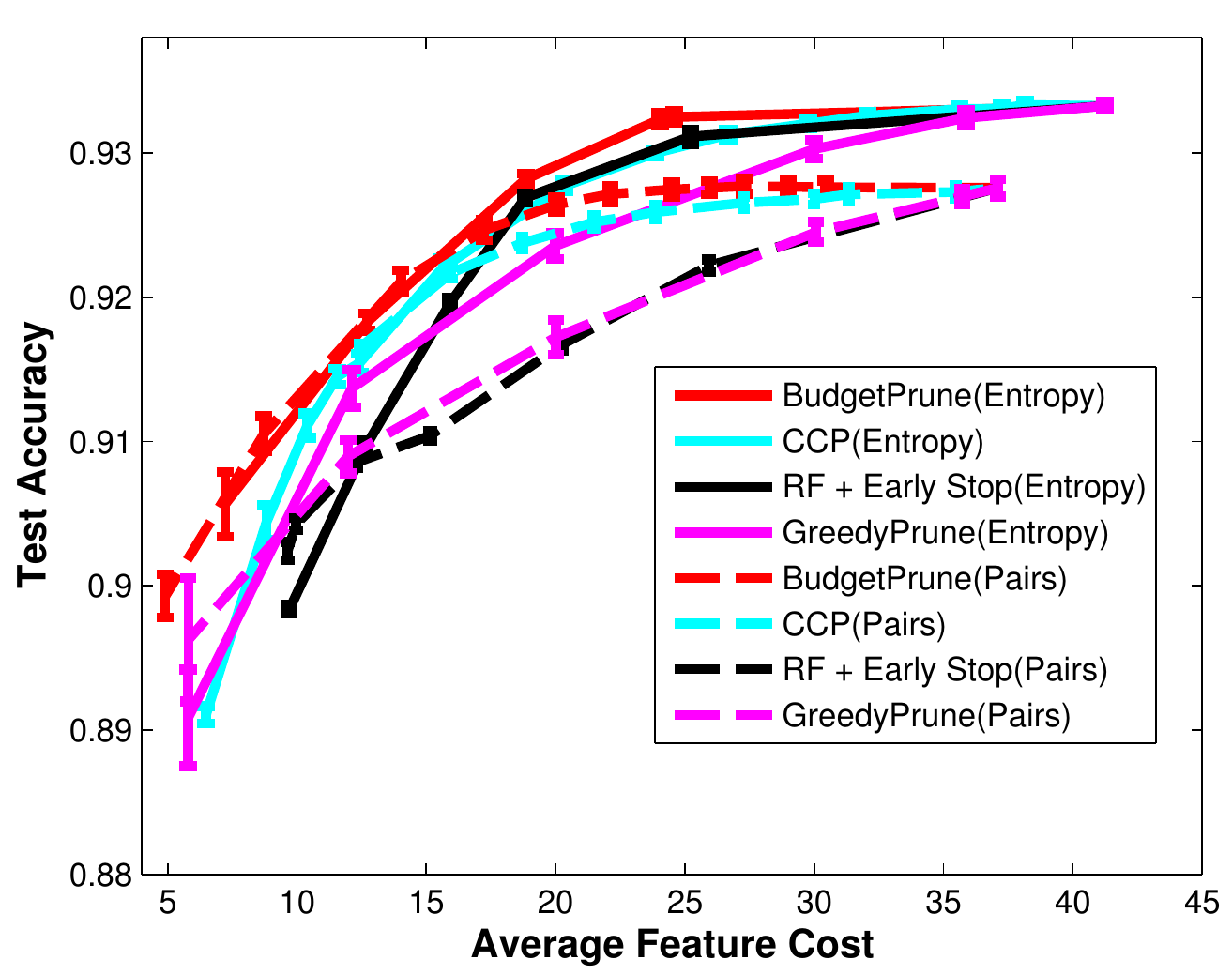}}
\subcaptionbox*{Forest Covertype}{\includegraphics[width=.48\linewidth,height=.4\linewidth]{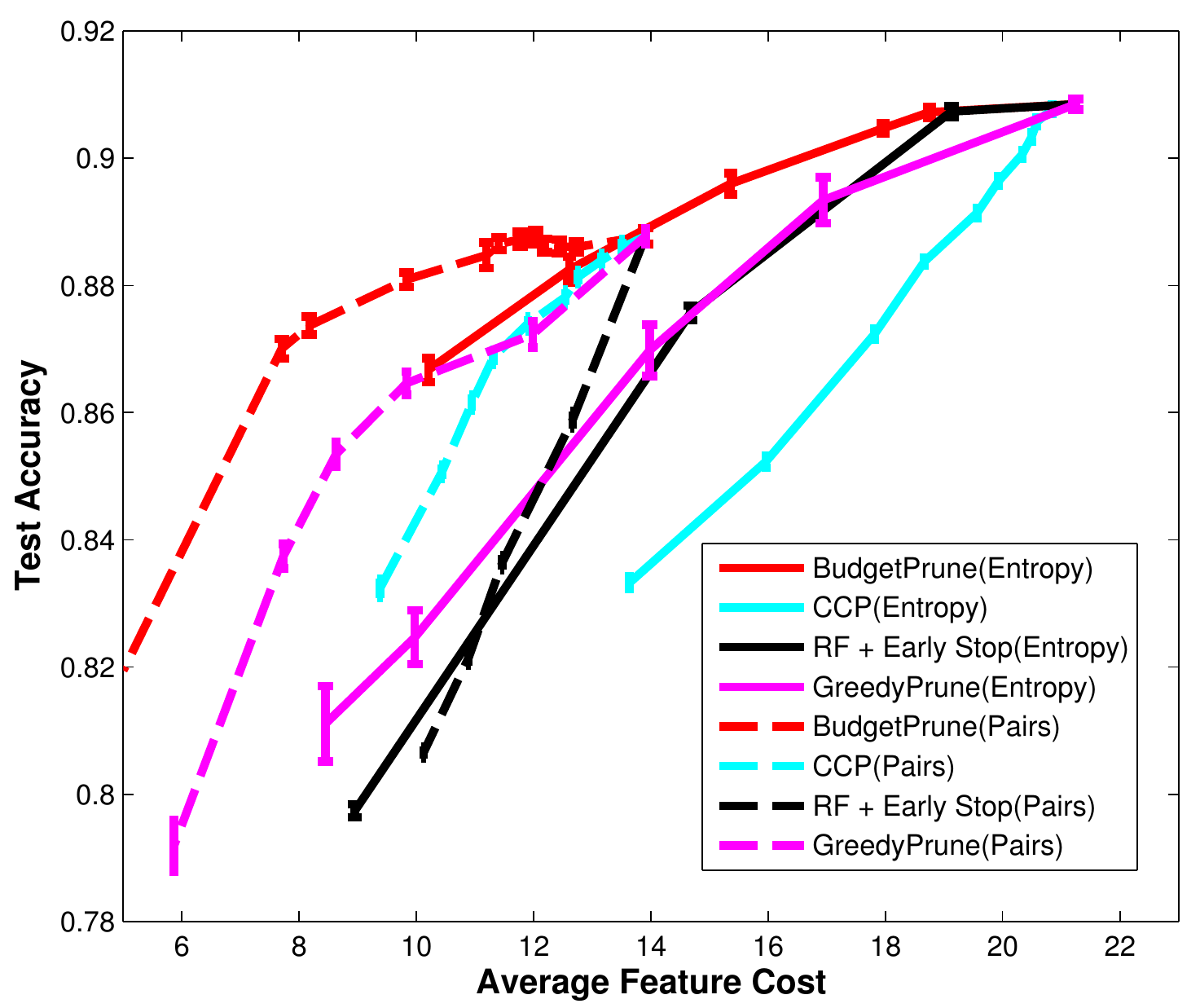}}
\caption{Comparisons of various pruning methods based on entropy and Pairs splitting criteria on MiniBooNE and Forest datasets} \label{fig:entropy_pair}
\end{figure}

\begin{figure}[htbp]
\centering
\includegraphics[width=0.5\textwidth]{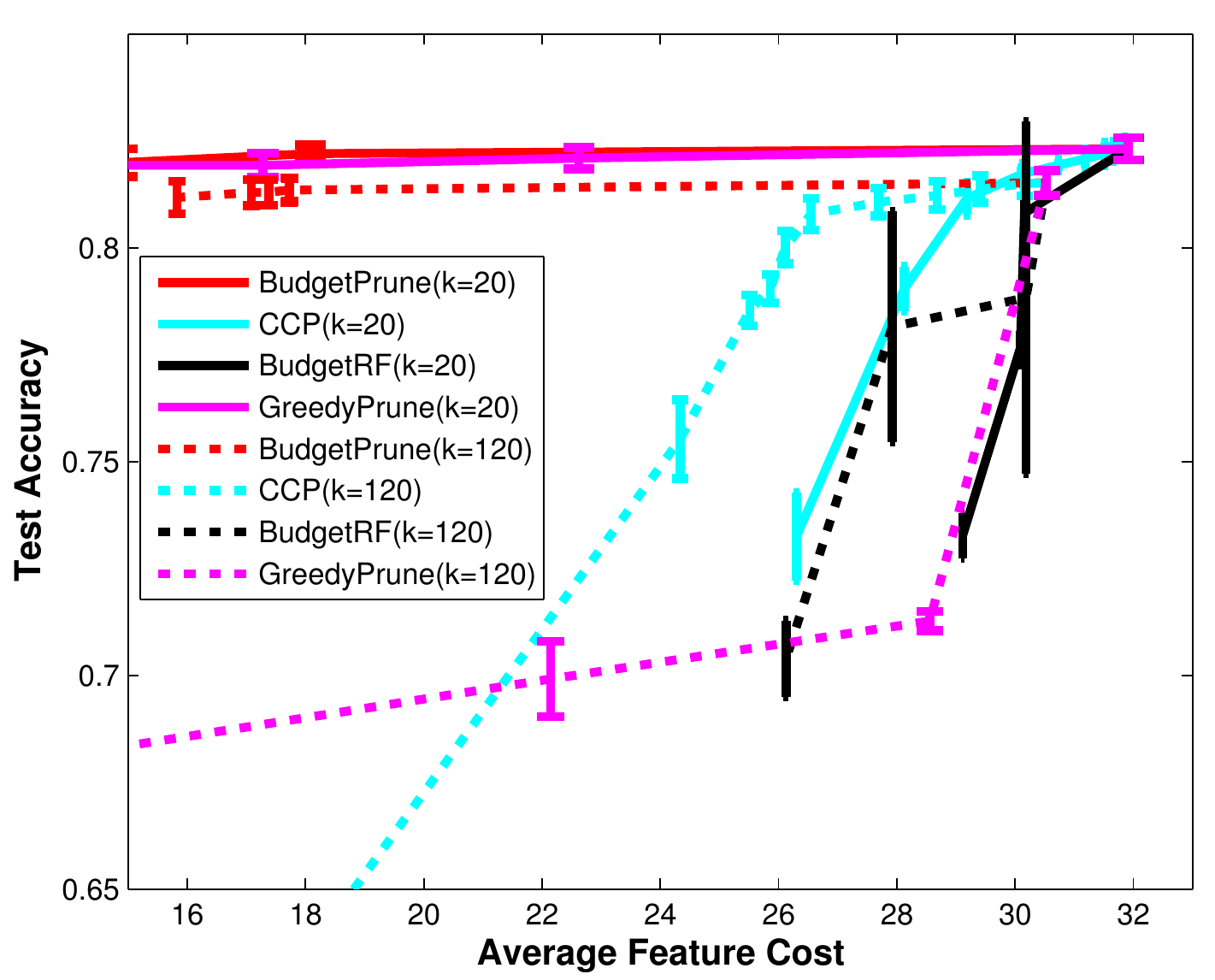}
\caption{Comparing various pruning approaches on RF built with k=20 and k=120 on Scene15 dataset. The initial RF has higher accuracy and higher cost for k=20. \textsc{GreedyPrune} performs very well in k=20 but very poorly in k=120.} \label{fig:scene15_k_20_120}
\end{figure}
\paragraph{Size of random feature subset at each split}
At each split in RF building, it is possible to restrict the choice of splitting feature to be among a random subset of all features. Such restriction tends to further reduce correlation among trees and gain prediction accuracy. The drawback is that test examples tend to encounter a diverse set of features, increasing feature acquisition cost. For illustration purpose, we plot various pruning results on Scene15 dataset for feature subset sizes $k=20$ and $k=120$ in Figure \ref{fig:scene15_k_20_120}. The initial RF has higher accuracy and higher cost for $k=20$ as expected. \textsc{BudgetPrune} achieves slightly better accuracy in $k=20$ than $k=120$. Note also how \textsc{GreedyPrune} performance drops significantly for $k=120$ so it is not robust. 
In our main experiments $k$ is chosen on validation data to achieve highest accuracy for the initial RF.

\clearpage
\bibliography{cost_sensitive_bib}
\bibliographystyle{plain}

\end{document}